\documentclass[11pt]{article}
\usepackage{pgm}

\usepackage{hyperref,color,soul,booktabs}
\setulcolor{blue}

\hypersetup{colorlinks,
            linkcolor=blue,
            citecolor=blue,
            urlcolor=magenta,
            linktocpage,
            plainpages=false}

\usepackage{xcolor}
\usepackage{enumitem}
\usepackage[ruled,linesnumbered]{algorithm2e}
\usepackage{mathtools}
\usepackage{tkz-graph}
\usepackage{subcaption}

\newtheorem{defn}{Definition}
\newtheorem{thrm}[defn]{Theorem}
\newtheorem{lem}[defn]{Lemma}

\newcommand{\an}{\mathrm{an}}
\newcommand{\de}{\mathrm{de}}
\newcommand{\pa}{\mathrm{pa}}
\newcommand{\ma}{\mathrm{ma}}
\newcommand{\cc}{\mathrm{cc}}
\newcommand{\ngh}{\mathrm{ne}}
\newcommand{\G}{\mathcal{G}}
\newcommand{\U}{\mathcal{U}}
\newcommand{\D}{\mathcal{D}}
\newcommand{\tg}{{<^\D}}

\renewcommand{\H}{\mathcal{R}}
\newcommand{\E}{\mathcal{E}}

\newcommand{\mathdash}{\relbar\mkern-9mu\relbar}
\newcommand{\rightopen}{\mathop{{\mathdash}\!{\circ}}}

\newcommand{\openopen}{\mathop{{\circ}\!{\mathdash}\!{\circ}}}
\newcommand{\at}{\mathrm{at}}
\newcommand{\mt}{\mathrm{mt}}

\setenumerate{noitemsep,topsep=0pt,parsep=0pt,partopsep=0pt}

\ShortHeadings{Transformational Characterization of Unconditional Equivalence}{Markham et al.}

\begin{document}

\title{A Transformational Characterization of Unconditionally Equivalent Bayesian Networks}

\author{Alex Markham \Email{alex.markham@causal.dev}\\
  \Name{Danai Deligeorgaki} \Email{danaide@kth.se}\\
  \Name{Pratik Misra} \Email{pratikm@kth.se}\\
  \Name{Liam Solus} \Email{solus@kth.se}\\
  \addr Department of Mathematics, KTH Royal Institute of Technology, Sweden}

\maketitle

\begin{abstract}
We consider the problem of characterizing Bayesian networks up to unconditional equivalence, i.e., when directed acyclic graphs (DAGs) have the same set of unconditional \(d\)-separation statements.
Each unconditional equivalence class (UEC) is uniquely represented with an undirected graph whose clique structure encodes the members of the class.
Via this structure, we provide a transformational characterization of unconditional equivalence; i.e., we show that two DAGs are in the same UEC if and only if one can be transformed into the other via a finite sequence of specified moves. 
We also extend this characterization to the essential graphs representing the Markov equivalence classes (MECs) in the UEC.  
UECs partition the space of MECs and are easily estimable from marginal independence tests. 
Thus, a characterization of unconditional equivalence has applications in methods that involve searching the space of MECs of Bayesian networks.
\end{abstract}


\section{Introduction}
\label{sec:intro}

A central aspect of modern causal discovery methods is the subdivision of the space of DAGs into equivalence classes that constitute the distinct possible representatives of the data-generating distribution.
When the available data is observational, the typical approach is to subdivide DAG space into \emph{Markov equivalence classes (MECs)} and then learn the \emph{essential graph (CPDAG)} representative of the class \citep{AMP97}. 
Popular causal discovery algorithms, such as the \emph{Greedy Equivalence Search (GES)}, move between candidate MECs based on a transformational relation between I-maps known as Meek's Conjecture \citep{meek1997graphical}, which was proven by \cite{chickering2002optimal}.  
This relation between I-maps is an extension of a transformational characterization of DAGs within the same MEC, which states that any two DAGs in the same MEC are connected by a sequence of covered arrow reversals \citep{chickering1995transformational}.
Although GES is known to be consistent, \citet{wille2006low} and \citet{markham2022} observed that lower-order conditional independencies still reflect well the causal structure of the model, while having the advantage of being accurately estimable even with few observations.  
This motivates a generalization of the transformational characterization of MECs underlying GES to \emph{unconditional equivalence classes (UECs)}; i.e., equivalence classes of DAGs defined by the marginal independence relations they encode.  
We provide such a transformational characterization of UECs both in terms of the DAGs they contain and their essential graphs.

\section{Results}
\label{sec:theory}

In the following we let $\D = (V^\D,E^\D)$ denote a directed acyclic graph (DAG) with set of nodes $V^\D$ and edge set $E^\D$ (and likewise \(\U\) for undirected graphs and \(\G\) for mixed graphs).  
When it is clear from context, we write $V$ and $E$ for the nodes and edges of $\D$, respectively. 
We let $D$ denote the skeleton of $\D$. 
Given disjoint $A,B,C\subseteq V$, we write $A\perp_\D B \mid C$ whenever $A$ and $B$ are \(d\)-separated given $C$ in $\D$.
Similarly, we write $A\not\perp_\D B \mid C$ whenever $A$ and $B$ are \(d\)-connected given $C$ in $\D$.  
A \emph{(simple) trek} between nodes $v$ and $w$ of $\D$ is a path $\pi = \langle v_1 = v,\ldots,v_m = w\rangle$ over distinct vertices that contains no \emph{colliders} $v_{i-1} \rightarrow v_i \leftarrow v_{i+1}$ as subpaths.
Two subsets $A$ and $B$ of $V$ are \(d\)-connected given $\emptyset$ if and only if there is a trek between some $v\in A$ and $w\in B$.  
A path $\pi = \langle v_1 ,\ldots,v_m \rangle$ is called \emph{blocked} if it contains a collider $v_{i-1} \rightarrow v_i \leftarrow v_{i+1}$.
Let $\pa_\D(A), \de_\D(A),$ and $\an_\D(A)$ denote the set of parents, descendants and ancestors of $A\subseteq V$ in $\D$, respectively.  
We note that $v\in\de_\D(v)$ and $v\in\an_\D(v)$.  
We say that an ordered pair $(v,w)$ is \emph{implied by transitivity} in $\D$ if $v\in\an_\D(w)\setminus(\{w\}\cup\pa_\D(w))$, and that the edge \(v \rightarrow w\) is implied by transitivity in the resulting graph \(\D_{+ v\rightarrow w}\) in which it is added.
The set of \emph{maximal ancestors} of $A$ in $\D$, denoted $\ma_\D(A)$ is the set of all $v\in \an_\D(A)$ for which $\an_\D(v) = \{v\}$.  
When the DAG $\D$ is understood from context, we simply write $\pa(A), \de(A), \an(A)$, and $\ma(A)$ for these sets.  
A node $v\in V$ is called a \emph{source} node of $\D$ if $\pa_\D(v) = \emptyset$.
It follows that $\ma_\D(V)$ is the collection of source nodes in $\D$.
A collection $\E$ of cliques (i.e., complete subgraphs) of an undirected graph $\U$ is called a \emph{minimal edge clique cover} of $\U$ if every edge of $\U$ is contained in at least one clique in $\E$ and no proper subset of $\E$ satisfies this property.  
An \emph{independent set} of a graph $\U$ is any subset of $V^\U$ in which no two nodes are adjacent.

\subsection{Unconditional Equivalence for Directed Acyclic Graphs}
\label{sec:transf-char-uncond}
Two DAGs $\D = (V,E^\D)$ and $\D' = (V,E^{\D'})$ are \emph{unconditionally equivalent} when for every \(i, j \in V\), we have that \(i \perp_\D j \mid \emptyset\) if and only if \(i \perp_{\D'} j \mid \emptyset \).
\citet[Lemma~5]{markham2022} show that unconditional equivalence is indeed an equivalence relation over ancestral graphs and consequently also over DAGs.
The collection of all DAGs that are unconditionally equivalent to $\D$ is called its \emph{unconditional equivalence class (UEC)}.
The \emph{unconditional dependence graph} of a DAG $\D = (V,E)$ is the undirected graph \(\U^\D = (V,\{\{v,w\} : v\not\perp_{\D} w; \;v, w  \in V\})\).
$\U^\D$ serves as a representative of the UEC of $\D$, analogous to the essential graph of an MEC.
When $\D$ is clear from context, we simply write $\U$ for $\U^\D$. 
We now provide two characterizations of the unconditional dependence graph of a UEC.  

\begin{thrm}
  \label{thm: first characterizations}
  The unconditional dependence graph $\U$ of a DAG $\D$ is equivalent to each of the following:
  \begin{enumerate}
  \item \(\U_1 = (V, \{\{v,w\} : \mathrm{an}_\D(v) \cap \mathrm{an}_\D(w) \not= \emptyset\})\), that is, two distinct nodes share an edge in \(\U\) if and only if they have a common ancestor in \(\D\);
  \item \(\U_2 = (V, \bigcup_{m \in \text{ma}_{\D}(V)} \{\{v, w\} \in \de_\D(m) \times \de_\D(m) : v \not= w \} )\).
  \end{enumerate}
\end{thrm}

\begin{proof}
  Two distinct nodes \(v, w\) are adjacent in \(\U\) if and only if they are \(d\)-connected given \(\emptyset\) in \(\D\), i.e., there is a trek between them.
  Such a trek must either be
  (i) directed from \(v\) to \(w\), implying \(v \in \mathrm{an}_\D(w)\), 
  (ii) directed from \(w\) to \(v\), implying \(w \in \mathrm{an}_\D(v)\), or
  (iii) consisting of a node \(c\) with directed paths to both \(v\) and \(w\), implying \(c \in \mathrm{an}_\D(v) \cap \mathrm{an}_\D(w)\).
  This happens if and only if \(\mathrm{an}_\D(v) \cap \mathrm{an}_\D(w) \not= \emptyset\), so \(\U = \U_1\).
  
  We now show that \(\U_2 = \U_1\). 
  Note that if \(\{v,w\}\in\U_1\) then \(v\neq w\) and \(\an_\D(v)\cap\an_\D(w)\neq \emptyset\).
  Since \(\D\) is a DAG, it follows that there exists \(m\in\an_\D(v)\cap\an_\D(w)\) such that  \(m\in \ma_\D(V)\).
  Since \(v\neq w\), it follows that \(\{v,w\}\in\U_2\).
  On the other hand, if \(\{v,w\}\in\U_2\), there exists a node \(m\) such that \(v,w\in\de_\D(m)\).
  It follows that \(m\in\an_\D(v)\cap\an_\D(w)\) and \(\{v,w\}\in\U_1\), which completes the proof.
\end{proof}

Note that Theorem~\ref{thm: first characterizations}~(2) implies that the cliques of a minimal edge clique cover of $\U$ are in bijection with the source nodes of $\D$. 
Namely, given $v\in\ma_\D(V)$, the maximal clique corresponding to $v$ is the clique $C_v := \de_\D(v)$, and $\E^\D :=\{C_v : v\in\ma_\D(V)\}$ is a minimal edge clique cover of $\U$. 
The following lemma will be used.

\begin{lem}
\label{lem:edge-clique}
Suppose $\D$ and $\D^\prime$ are two DAGs in the same UEC with unconditional dependence graph $\U$.  
Then:
\begin{enumerate}
    \item $|\ma_\D(V)| = |\E^\D|$,
    \item $\ma_\D(V)$ is a maximum independent set in $\U$,
    \item $\E^\D = \E^{\D^\prime}$, and 
    \item $|\ma_\D(V)| = |\ma_{\D^\prime}(V)|$.
\end{enumerate}
\end{lem}

\begin{proof}
Since each $C\in\E^\D$ contains a unique node $v\in\ma_\D(V)$, it follows that $|\E^\D| = |\ma_\D(V)|$. 
It also follows from the construction of $\U_2$ in Theorem~\ref{thm: first characterizations}~(2) that $\ma_\D(V)$ is an independent set in $\U$.  
To see that $\ma_\D(V)$ is a maximum independent set in $\U$, suppose there exists $I\subseteq V$, an independent set with $|I|>|\ma_\D(V)|$.  
By Theorem~\ref{thm: first characterizations}~(2), every $v\in V$ is in at least one clique in $\E^\D$.  
Since $|I|>|\ma_\D(V)| = |\E^\D|$, it follows that at least one clique $C\in\E^\D$ contains two nodes in $I$, a contradiction.  
Hence, $\ma_\D(V)$ must be a maximum independent set in $\U$.  

Since $\ma_{\D^\prime}(V)$ is also an independent set in $\U$, it follows that there is precisely one $v\in\ma_{\D^\prime}(V)$ in each $C\in\E^\D$.  
Let $C'_v$ be the (unique) clique in $\E^{\D^\prime}$ containing $v\in\ma_{\D^\prime}(V)$, and let $C_v$ be the (unique) clique in $\E^\D$ containing $v$.  
It follows that $C'_v\subseteq C_v$.  
To see this, recall that $C_v$ contains a unique node $w\in\ma_\D(V)$ such that $C_v = \de_\D(w)$.  
Since $v\in C_v$, the nodes $v$ and $w$ are adjacent in $\U$, which along with $v\in\ma_{\D^\prime}(V)$ gives us, by Theorem~\ref{thm: first characterizations}~(2), that $w\in\de_{\D^\prime}(v)$.  
Since any $k\in C'_v$ is also a descendant of $v$ in $\D^\prime$, it follows that there is a trek between $k$ and $w$ in $\D^\prime$.
Hence, $k\not\perp_{\D^\prime} w$ and so $k$ and $w$ are adjacent in $\U$.  
Since $C_v$ is the maximal clique in $\U$ containing $w$, it follows that $C'_v\subseteq C_v$.  
By symmetry of the argument, $C'_v = C_v$.  
Thus, $\E^\D = \E^{\D^\prime}$.  
It follows immediately that $|\ma_\D(V)| = |\ma_{\D^\prime}(V)|$.
\end{proof}

Since $\E^\D = \E^{\D^\prime}$ for all $\D,\D^\prime$ in the UEC represented by $\U$, we can set $\E^\U := \E^\D$ for any $\D$ in the class and $\E^\U$ will be well-defined.

\begin{defn}
An edge $v\rightarrow w$ in a DAG $\D$ is \emph{weakly covered} in \(\D\) if \(\ma(\pa(v)) = \ma(\pa(w)\setminus\{v\})\) and \(v \not\in \an(\pa(w)\setminus \{v\})\).
An ordered pair $(v,w)$ of nonadjacent nodes $v,w\in V^\D$ is \emph{partially weakly covered} if \(\ma(v) \subseteq \ma(w)\), \(v \not\in \an(w)\), \(w \not\in \an(v)\), and \(\pa(v) \not= \emptyset\).
\end{defn}

Observe that, when $pa(v)\neq \emptyset$, removing a weakly covered edge $v\rightarrow w$ results in a partially weakly covered pair $(v,w)$, in which case we could also say that \(v\rightarrow w\) is a \emph{a partially weakly covered edge}.
We say that a partially weakly covered pair $(v,w)$ is \emph{strictly partially weakly covered} if \(\ma(v) \varsubsetneq \ma(w)\).
Weakly covered edges generalize the \emph{covered edges}, $v\rightarrow w$ with $pa(v)= pa(w)\setminus\{v\}$, used in the transformational characterization of Markov equivalence \citep{chickering1995transformational}. 
Notice also that the condition $v \not\in \an(\pa(w))$ in the definition of partially weakly covered excludes the possibility of a partially weakly covered edge also being implied by transitivity.  
These two conditions combine to characterize those edges that can be added to a DAG to produce another DAG in the same UEC.  

\begin{lem}
\label{lemma:add-edge}
Let \(\D\) be a DAG with nonadjacent nodes $v$ and $w$, and let $\D^\prime$ be the identical digraph but with the edge $v\rightarrow w$ added. 
  Then \(\D'\) is a DAG unconditionally equivalent to \(\D\) if and only if \((v, w)\) is partially weakly covered or implied by transitivity in \(\D\). 
\end{lem}
\begin{proof}
  Assume \((v, w)\) is partially weakly covered in \(\D\).
  Because $w$ is not an ancestor of $v$, $\D^\prime$ does not contain any cycles.  
 Since $\ma_\D(v) \subseteq \ma_\D(w)$ and $\pa_\D(v)\neq \emptyset$, it also follows that for every parent of \(v\) there is a trek to \(w\) in $\D$.
  This means there can be no \(v'\) for which all paths to \(w\) are blocked by \(v\).
  Thus, adding the edge \(v\rightarrow w\) neither creates nor removes existing $d$-connecting paths given $\emptyset$, so \(\U^\D = \U^{\D'}\). 

  Alternatively, assume that \((v, w)\) is implied by transitivity in \(\D\), i.e., \(v \not\in \pa(w)\) but \(v \in \an(w)\).
  First, observe that because \(\D\) is a DAG, any cycle in \(\D'\) must contain \(v\rightarrow w\).
  However, \(v \in \an_\D(w)\), and hence there can be no cycle in \(\D'\) unless there is a cycle in \(\D\).
  Observe that \(\D'\) has the same set of ancestral relations as \(\D\) (i.e., adding a transitively implied edge does not change the ancestor set of any node), so by Theorem \ref{thm: first characterizations}~(1), \(\U^\D = \U^{\D'}\).
  
  Conversely, assume that \((v, w)\) is neither partially weakly covered nor transitively implied in \(\D\).
  There are two cases: either \(v \perp_\D w\) or \(v \not\perp_\D w\).
  In the first case, \(v\rightarrow w \in E^{\D'}\) implies \(v \mathdash w \in E^{\U^{\D'}}\) and thus \(\U^{\D'} \not= \U^\D\).
  In the second case, there must exist a trek between $v$ and $w$ in $\D$.  
  Hence, either (i) $w\in \an_\D(v)$,
  (ii) $v\in \an_\D(w)$, or
  (iii) for every trek, there exists $m\in \an_\D(v)\cap \an_\D(w)$ where $m\notin\{v,w\}$.  
  In subcase (i), $\D'$ would contain a cycle and thus would not be a DAG. 
  In subcase (ii), $v\in \an_\D(w)$ contradicts our assumption that $(v, w)$ is not implied by transitivity in $\D$.  
  In subcase (iii), $q\in$ $\pa_\D(v)\neq \emptyset$.  
  Hence, because $(v, w)$ is neither partially weakly covered nor implied by transitivity in $\D$ and $w\notin \an_{\D}(v)$, there must exist $m'\in\ma_\D(v)$ such that $m'\notin\ma_\D(w)$.
  Hence $\mathrm{an}_\D(m')\cap\mathrm{an}_\D(w) = \emptyset$, since $m'$ is a maximal ancestor in $\D$. By Theorem \ref{thm: first characterizations}~(1), $m'\perp_{\D}w$. 
  However, the addition of the edge $v\rightarrow w$ implies that $m'\in\an_{\D'}(w)$, thus $m'\not\perp_{\D'}w$ and hence $\D$ and $\D'$ are not unconditionally equivalent.
\end{proof}

Analogous to covered edges for MECs, weakly covered edges are precisely the edges that can be reversed to move between DAGs in the same UEC.  

\begin{lem}
\label{lemma:rev-edge}
  Let \(\D\) be a DAG containing the edge \(v\rightarrow w\), and let \(\D'\) be the digraph identical to \(\D\) but with the edge reversed, so \(w\rightarrow v \in E^{\D'}\).
  Then \(\D'\) is a DAG that is unconditionally equivalent to \(\D\) if and only if \(v\rightarrow w\) is weakly covered in \(\D\). 
\end{lem}

\begin{proof}
  This proof is similar to that of Lemma \ref{lemma:add-edge} as well as the (non-weakly) covered edge case for Markov equivalence \citep[Lemma 1]{chickering1995transformational}.
  Assume \(v\rightarrow w\) is weakly covered in \(\D\).
  First, observe that \(\D'\) is a DAG: in order for the reversed edge \(w\rightarrow v\) to form a cycle in \(\D'\), it would require \(v \in \an_\D(\pa_\D(w)\setminus\{v\})\), which is by definition not allowed for a weakly covered edge.
  The definition of weakly covered also implies that for every ancestor of $v$  there is a trek to $w$ in $\D^{\prime}$;
  similarly, every ancestor of $w$ is $d$-connected to $v$ given $\emptyset$ in $\D$.
  Thus, reversing edge \(v\rightarrow w\) to \(w\rightarrow v\) neither creates nor removes any existing $d$-connecting paths given $\emptyset$, so \(\U^\D = \U^{\D'}\).
  
  Conversely, assume \(v\rightarrow w\) is not weakly covered in \(\D\).
  We show that either \(\D'\) contains a cycle or \(\D'\) is a DAG that is not unconditionally equivalent to \(\D\).
  There are three cases:
  (i) \(v \in \an(\pa_\D(w)\setminus \{v\})\), 
  (ii) there is some \(c \in \ma(\pa_\D(v))\) but \(c \not\in \ma(\pa_\D(w)\setminus\{v\})\) such that \(v\) lies in all paths from \(c\) to \(w\) in \(\D\), or
  (iii) there is some \(c \in \ma(\pa_\D(w)  \setminus\{v\})\) but \(c \not\in \ma(\pa_\D(v))\), in which case all paths from \(c\) to \(v\) in \(\D\) are blocked by $w$ or other descendants of $v$.
  In case (i), reversing the edge \(v\rightarrow w\) creates a cycle in \(\D'\).
  In case (ii), reversing the edge blocks all paths from $c$ to $w$, introducing a new unconditional independence in $\D'$. 
  In case (iii), reversing the edge \(v\rightarrow w\) produces a $d$-connecting path from $c$ to $w$ given $\emptyset$, removing an unconditional independence from $\D$.
  Thus, in any case, \(\U^\D \not= \U^{\D'}\).
\end{proof}

Via Lemmas~\ref{lemma:add-edge} and~\ref{lemma:rev-edge}, we can show that any two DAGs $\D$ and $\D^\prime$ in the same UEC are connected by a sequence of edge additions, reversals, and removals, such that after each transformation, the resulting DAG is also in the UEC.  
To do so, we first show that one can produce a DAG $\H^{\D,\D^\prime}$ in the UEC for which two nodes are adjacent in $\H^{\D,\D^\prime}$ if and only if they are adjacent in either $\D$ or $\D^\prime$.  

\IncMargin{1em}
\begin{algorithm}
  \SetKwInOut{Input}{input}\SetKwInOut{Output}{output}
  \Input{unconditionally equivalent DAGs \(\D\) and \(\D'\)}
  \Output{DAG \(\H\) in the same UEC as $\D$ and $\D^\prime$}
  \BlankLine
  \(\tg\) := a total ordering of $V^\D$ : $v\tg w$ if $(v,w)$ strictly partially weakly covered\\
 $\H:=\D$\\
  \For{every pair \(\{v, w\}\) of nodes adjacent in \(\D'\) but not \(\H\)}{
    \If{\(v \tg w\)}{add edge \(v\rightarrow w\) to \(\H\)}
    \Else{add edge \(w\rightarrow v\) to \(\H\)}}
  \Return{\(\H\)}
  \caption{\texttt{find\_reference}(\(\D, \D'\))}
  \label{alg:find_ref}
\end{algorithm}\DecMargin{1em}

Let $\H^{\D,\D^\prime}$ denote the output of Algorithm~\ref{alg:find_ref} given $\D$ and $\D^\prime$. 
Note that $\H^{\D,\D^\prime}$ is a DAG.  

\begin{lem}
\label{lem: reference graph}
  Given two unconditionally equivalent DAGs \(\D\) and  \(\D'\), for every edge \(v\rightarrow w\) in $\H^{\D,\D^\prime}$ that is not in \(\D\), the pair \((v, w)\) is either implied by transitivity or partially weakly covered in \(\D\).
\end{lem}
\begin{proof}
  Let $\hat{\U}$ denote the graph $\U$ with all edges $a \mathdash b$ removed for which there exists a clique in $\E^\U$ containing both $a$ and $b$, a clique in $\E^\U$ containing $a$ but not $b$, and a clique in $\E^\U$ containing $b$ but not $a$.
  
Suppose now that $v\rightarrow w, u\rightarrow t\notin E^{\D}$ and $v  <^{\D} w$ and $u  <^{\D}t$, with $(v,w), (u,t)$ each partially weakly covered or implied by transitivity in $\D = (V,E)$.
Then $(u,t)$ is also partially weakly covered or implied by transitivity in $\overline{\D} = (V,E\cup\{v\rightarrow w\})$.  
To see this, assume first that $(u,t)$ is implied by transitivity in $\D$, i.e., $u\in\an_\D(t)$. Since
adding an edge to $\D$ implies that $\an_{\D}(S) \subseteq \an_{\overline{\D}}(S)$ for all $S\subseteq V$, $(u,t)$ also satisfies  $u\in\an_{\overline{\D}}(t)$.  

Since $(v,w)$ is partially weakly covered or implied by transitivity in $\D$, we have $\ma_{\overline{\D}}(w)=\ma_{\D}(w)$, and hence $\ma_{\overline{\D}}(S)=\ma_\D(S)$ for any set $S\subseteq V$.
Hence, $\ma_{\overline{\D}}(u)= \ma_{\D}(u)$ and $\ma_{\overline{\D}}(t)= \ma_{\D}(t)$.  
If $(u,t)$ is partially weakly covered in $\D$ then it follows by  $\ma_{\D}(u) \subseteq \ma_{\D}(t)$ and $\pa_{\D}(u)\neq \emptyset$ that $(u,t)$ is partially weakly covered in $\overline{\D}$ in case $u\notin \an_{\overline{\D}}(t)$, and implied by transitivity in case $u\in \an_{\overline{\D}}(t)$. 
This observation shows that we can add edges that are implied by transitivity or partially weakly covered in $\D$ to $\D$ in any order, and after each addition, by Lemma~\ref{lemma:add-edge}, the resulting graph is a DAG that is unconditionally equivalent to $\D$.  
Let $\tilde{\U}$ denote the DAG resulting from all such edge additions.  
Note that the ordering $<^\D$ chosen in Algorithm~\ref{alg:find_ref} does not affect the skeleton of $\tilde{\U}$, but only the direction of the arrows added in $\D$.
Here we use the assumption that the topological ordering $<^\D$ agrees with the topological ordering of $\tilde{\U}$.

We now show that the set of edges in $\U$ that are not in the skeleton of $\tilde{\U}$ is equal to the set of edges in $\U$ that are not in $\hat{\U}$, implying that the skeleton of $\tilde{\U}$ is equal to $\hat{\U}$.  
Let $u \mathdash v$ be an edge of $\U$. By Theorem~\ref{thm: first characterizations}~(1) there exists $m \in \ma_{\D} (V)$ such that $m\in\an_\D(u)\cap\an_\D(v)$.
Hence, $u$ and $v$ both belong to the maximal clique of $ \E^\U$ that contains $m$.

If  $u\rightarrow v$ or $v\rightarrow u$ is in $\D$ then $u$ and $v$ are adjacent in $\tilde\U$, by definition. Moreover, either $\ma_\D(u)\subseteq \ma_\D(v)$ or $\ma_\D(v)\subseteq \ma_\D(u)$.
Hence, Theorem~\ref{thm: first characterizations}~(2) implies that the edge $u \mathdash v$ is not removed from $\U$ when constructing $\hat\U$.

Now let $u\rightarrow v, v\rightarrow u\notin E^\D$.
Then, $u$ and $v$ are not adjacent in $\tilde{\U}$ if and only if $(u,v)$ and $(v,u)$ are neither implied by transitivity in $\D$ nor partially weakly covered in $\D$.
By definition, $(u,v)$ and $(v,u)$ are neither implied by transitivity in $\D$ nor partially weakly covered in $\D$ if and only if $u\notin\an_{\D}(v)$, $v\notin\an_{\D}(u)$ and
$\ma_\D(u)\not\subseteq\ma_\D(v)$,
$\ma_\D(v)\not\subseteq\ma_\D(u)$, i.e., there exist $m^\prime\in\ma_\D(u), m^{\prime\prime}\in\ma_\D(v)$ such that $m^\prime \notin \ma_\D(v), m^{\prime\prime} \notin \ma_\D(u)$.
(Note that $u$ and $v$ each have at least one parent in $\D$, by the assumption $u\mathdash v\in E^{\U}$, when neither $(u,v)$ or $(v,u)$ is implied by transitivity).
Hence, it follows from Theorem~\ref{thm: first characterizations}~(2) that $u$ and $v$ are not adjacent in $\tilde{\U}$ if and only if $u\notin\an_{\D}(v)$, $v\notin\an_{\D}(u)$ and $u$ is contained in the maximal clique of $\E^\U$ containing $m^\prime$ but $v$ is not, and similarly $v$ is contained the maximal clique of $\E^\U$ containing $m^{\prime\prime}$ but $u$ is not.
Since $u \mathdash v \in E^{\U}$ implies that both $u$ and $v$ belong to the maximal clique of $\E^\U$ containing $m$, we have that $u$ and $v$ are not adjacent in $\tilde{\U}$ if and only if $u$ and $v$ are not adjacent in $\hat{\U}$.

It follows that for any graph $\D$ in the UEC represented by $\U$, the edges in $\hat{\U}$ that are not in the skeleton of $\D$ are implied by transitivity in $\D$ or are partially weakly covered in $\D$.  
Let $\D^\prime$ be another member of the UEC of $\U$.  
Suppose that $u,v$ are not adjacent in $\D$ but are adjacent in $\D^\prime$.  
Then $u \mathdash v$ is in the skeleton of $\tilde{\U}$.
Hence, if $u<^\D v$, then $(u,v)$ is either implied by transitivity or partially weakly covered in $\D$ (since adding all such edges in any order produces the graph $\H^{\D,\D^\prime}$), which completes the proof.
\end{proof}

We say that a DAG $\D$ is \emph{maximal} in its UEC if any DAG produced by adding an edge to $\D$ is not in the same UEC as $\D$. Note that the DAG $\tilde{\U}$ constructed in the proof of Lemma \ref{lem: reference graph} is maximal for any choice of DAG $\D$.

\begin{defn}
  Let \(\Delta(\D, \D') \coloneqq \{v\rightarrow w \in E^\D : w\rightarrow v \in E^{\D'}\}\) denote the set of edges in \(\D\) that have opposite orientation in \(\D'\).
  Let \(\Gamma(\D, \D') \coloneqq \{v\rightarrow w \in E^{\D'} : v\rightarrow w \not\in \Delta(\D',\D) \cup E^{\D}\}\) denote the set of edges between nodes adjacent in \(\D'\) but not \(\D\).
\end{defn}

Algorithm~\ref{alg:find_edge} provides a generalization of the \texttt{find edge} algorithm of \citep{chickering1995transformational} for the identification of covered edges to an algorithm that identifies weakly covered edges.  

\IncMargin{1em}
\begin{algorithm}
  \SetKwInOut{Input}{input}\SetKwInOut{Output}{output}
  \Input{unconditionally equivalent DAGs \(\D, \D'\) with the same skeletons but at least one differently oriented edge}
  \Output{edge from \(\Delta(\D, \D')\)}
  \BlankLine
  \(\tg\):= a total ordering of $V^\D$ \\
  let \(w\) be the minimal node with respect to \(\tg\) for which there is \(w'\) such that \(w'\rightarrow w \in \Delta(\D, \D')\)\\
  let \(v\) be the maximal node with respect to \(\tg\) for which \(v\rightarrow w \in \Delta(\D, \D')\)\\
  \Return{\(v\rightarrow w\)}
  \caption{\texttt{find\_edge}(\(\D, \D'\)); adapted from \citep{chickering1995transformational}}
  \label{alg:find_edge}
\end{algorithm}\DecMargin{1em}

\begin{lem}
\label{lem:rev-edge}
  The edge \(v\rightarrow w\) returned by Algorithm \ref{alg:find_edge} is weakly covered in \(\D\).
\end{lem}

\begin{proof}
  There are two cases:
  either (i) \(\ma(\pa_\D(v)) = \emptyset\), or 
  (ii) there is some \(v' \in \ma(\pa_\D(v))\).
  In case (i), it must also be that \(\ma(\pa_{\D}(w) \setminus\{v\}) = \emptyset\), otherwise \(v\) would be \(d\)-separated from all other ancestors of \(w\)  given $\emptyset$ in $\D$, contradicting the fact that both \(\D'\) has edge \(w\rightarrow v\) (producing a $d$-connecting path between $v$ and the ancestors of $w$  given $\emptyset$ in \(\D'\)) and \(\D'\) is unconditionally equivalent to \(\D\) with the same skeleton.
  Hence, \(\ma(\pa_\D(v)) = \ma(\pa_\D(w) \setminus\{v\}) = \emptyset\) and \(v \not\in \an(\pa_\D(w)\setminus \{v\})\), so \(v \rightarrow w\) is weakly covered in \(\D\).
  In case (ii), for every \(v' \in \ma(\pa_\D(v))\), there must also be a directed path in $\D$ from \(v'\) to \(w\) that does not contain \(v\); otherwise, since \(w\) is minimal with respect to \(\tg\), the edge \(w\rightarrow v\) in \(\D'\) would make \(v\) a collider, and \(v'\) would be \(d\)-separated from \(w\) given $\emptyset$ in $\D^{\prime}$.
  Hence, \(\ma(\pa_\D(v)) \subseteq \ma(\pa_\D(w) \setminus\{v\})\).
  Furthermore, using the same argument as in case (i), there can be no additional maximal ancestors of \(w\), so \(\ma(\pa_\D(v)) \supseteq \ma(\pa_\D(w) \setminus\{v\})\).
  Finally, because \(v\) is maximal with respect to \(\tg\), and because \(\D'\) is acyclic and unconditionally equivalent to \(\D\) and with the same skeleton, \(v \not\in \an(\pa_{\D}(w)\setminus \{v\})\). 
  Hence \(v\rightarrow w\) is weakly covered in \(\D\).
\end{proof}

The next theorem generalizes the transformational characterization of Markov equivalence \citep{chickering1995transformational} to a transformational characterization of unconditional equivalence.

\begin{thrm}[Transformational Characterization]
\label{thm: transformational}
  Let $\D$ and $\D^\prime$ be two unconditionally equivalent DAGs.  
  There exists a sequence of $|E^{D^\prime}\setminus E^{D}|$ edge insertions, followed by $|\Delta(\H^{\D, \D'}, \H^{\D', \D})|$ edge reversals, followed by $|E^{D}\setminus E^{D^\prime}|$ edge deletions that transforms $\D$ into $\D^\prime$ with the following properties:
  \begin{enumerate}
  \item Each edge inserted or deleted in \(\D\) is partially weakly covered or implied by transitivity.
  \item Each edge reversed in \(\D\) is weakly covered.
  \item After each operation, the resulting \(\D\) is a DAG and \(\U^\D = \U^{\D'}\).
  \item After all operations, \(\D = \D'\).
  \end{enumerate}
\end{thrm}

\begin{proof}
Using Algorithm~\ref{alg:find_ref}, $\H^{\D,\D^\prime}$ is produced from $\D$ by adding exactly the edges in $\Gamma(\D,\H^{\D,\D^\prime})$, of which there are $|E^{D^\prime}\setminus E^{D}|$ in total. 
By Lemma~\ref{lem: reference graph}, each edge added in this phase is either partially weakly covered or implied by transitivity in $\D$. 
Similarly, $\H^{\D^\prime,\D}$ can be produced via Algorithm~\ref{alg:find_ref} by adding a sequence of $|E^{D}\setminus E^{D^\prime}|$ edges to $\D^\prime$ that are all either partially weakly covered or implied by transitivity in $\D^\prime$.
Moreover, as seen in the proof of Lemma~\ref{lem: reference graph}, if $\H$ is produced from $\H^\prime$ by adding $v\rightarrow w$ to $\H^\prime$ where $(v,w)$ is implied by transitivity or partially weakly covered in $\H^\prime$ then all other such pairs $(v,w)$ in $\H^\prime$ are still implied by transitivity or partially weakly covered in $\H$.  
Thus, reversing the sequence of edge additions used to produce $\H^{\D^\prime,\D}$ from $\D^\prime$, the edge removed from $\H^\prime$ to produce $\H$ is always partially weakly covered or implied by transitivity in $\H$.  
It further follows from Algorithm~\ref{alg:find_ref} that after each edge addition or removal, the resulting graph is a DAG.

Finally, consider $\H^{\D,\D^\prime}$ and $\H^{\D^\prime,\D}$.
By construction, their skeletons satisfy $R^{\D,\D^\prime} = R^{\D^\prime,\D}$.  
Starting with $\H^{\D,\D^\prime}$, by Lemma~\ref{lem:rev-edge}, Algorithm~\ref{alg:find_edge} identifies an edge in $\Delta(\H^{\D,\D^\prime},\H^{\D^\prime,\D})$ that is weakly covered, say $u\rightarrow v$.  
By Lemma~\ref{lem:rev-edge}, reversing this edge in $\H^{\D,\D^\prime}$ produces a DAG $\H$ for which $\Delta(\H,\H^{\D^\prime,\D}) = \Delta(\H^{\D,\D^\prime},\H^{\D^\prime,\D})\setminus\{u\rightarrow v\}$. 
Thus, after each reversal, the resulting graph is a DAG, and the cardinality of $\Delta(\H,\H^{\D^\prime,\D})$ is reduced by one.
\end{proof}

\subsection{Unconditional Equivalence for Essential Graphs}
\label{sec:uncond-equiv-cpdags}

An essential graph (CPDAG) is a chain graph that represents a Markov equivalence class of DAGs (see \cite{AMP97} for more background).
A given chain graph \(\G\) determines a preorder \((V^\G, \leq^\G)\) over its vertices (note that we omit the superscript \(^\G\) when it is clear from context), and \(v \leq w\) means either that \(v\) and \(w\) are identical or that there is an undirected or \emph{partially directed path} (i.e., a path with at least one directed edge and all directed edges pointing the same direction) from \(v\) to \(w\).
Nodes \(v\) and \(w\) are in the same \emph{chain component} 
if \(v \leq w\) and \(w \leq v\);
nodes in the same chain component must be connected by an undirected path;
and we denote the component of node \(v\) as \(\cc(v) \coloneqq \{w \in V^\G: v \leq w \text{ and } w \leq v\}\).
Generalizing the notions of ancestor and maximal ancestor in DAGs, the \emph{anterior} of a node \(v\) in a chain graph \(\G\) is defined as \(\at(v) \coloneqq \{w \in V^\G : w \leq v\}\), and we define the \emph{minimal anterior} as \(\mt(v) \coloneqq \{w \in \at(v) : \at(w) = \cc(w)\}\).
We use $[\G]$ to denote the MEC determined by $\G$, and we specify a DAG in this MEC using \(\D \in [\G]\).
In the following, we write $v\rightopen w$ to denote an edge that is either $v\rightarrow w$ or undirected, and \(\G_{- v \rightopen w}\) to denote the graph $\G$ with the edge $v\rightopen w$ removed. 

\begin{defn}
  An edge \(v \rightopen w\) of an essential graph \(\G\) is \emph{removable} when there exists a DAG \(\D \in [\G]\) such that the edge \(v \rightarrow w\in V^\D\) is either partially weakly covered or implied by transitivity in \(\D\).
\end{defn}

\begin{thrm}\label{thrm:essential}
  The edge \(v \rightopen w\) of an essential graph \(G\) is removable if and only if \\ \(\mt_{\G_{- v \rightopen w}}(v) \subseteq \mt_{\G_{- v \rightopen w}}(w)\).
\end{thrm}
\begin{proof}
  Assume that \(\mt(v)_{\G_{- v \rightopen w}} \subseteq \mt(w)_{\G_{- v \rightopen w}}\).
  There are two cases: either
  (i) \(w \leq v\) or 
  (ii) \(w \not\leq v\) in \(\G_{- v \rightopen w}\).

  In case (i), \(w \leq v\) implies there is either an undirected or partially directed path from \(w\) to \(v\) in \(\G_{- v \rightopen w}\).
  Such a partially directed path in \(\G_{- v \rightopen w}\) would mean \(\G\) (which contains an edge \(v \rightopen w\)) contains a partially directed cycle contradicting that \(\G\) is a CPDAG, so the path in \(\G_{- v \rightopen w}\) must be undirected, and furthermore \(\G\) must contain the edge \(v \mathdash w\), with \(v\) and \(w\) being in the same chain component.
 Because the induced subgraph over \(\cc_\G(\{v, w\})\) is chordal, there must be a node \(z\) such that the undirected path \(p_z = \langle v, z, w \rangle\) is in \(\G\).
  Recall that any DAG obtained by orienting the undirected edges of a CPDAG according to a perfect ordering, which can be obtained via the Max Cardinality Search (MCS) algorithm \cite[Algorithm 9.3]{koller2009probabilistic}, will be Markov equivalent to that CPDAG.
  Observe that we can start MCS at node \(v\) and then get \(z\) followed by \(w\), implying there must be a DAG \(\mathcal{D} \in [\G]\) in which \(p_z\) becomes the the directed path \(v \rightarrow z \rightarrow w\), rendering \(v \rightarrow w\) implied by transitivity in \(\mathcal{D}\).
  Thus, \(v \rightopen w\) is removable in \(\G\).

  In case (ii), \(\mt_{\G_{- v \rightopen w}}(v) \subseteq \mt_{\G_{- v \rightopen w}}(w)\) implies that \(v\) and \(w\) are connected by at least one partially directed trek.
  Consider the set of all induced partially directed treks between \(v\) and \(w\).
  Note that \(w \not\leq v\) prohibits any of these treks from being partially directed from \(w\) to \(v\).
   For any DAG \(\mathcal{D} \in [\G]\), each of these treks becomes fully directed as either (ii.a) \(v \rightarrow \cdots \rightarrow w\) or (ii.b) \(v \leftarrow \cdots \rightarrow w\).
   If any of these treks are of the form (ii.a), then \(v \rightarrow w\) is implied by transitivity in \(D\) and thus removable in \(\G\), otherwise all of these treks are of the form (ii.b) and thus \(v \rightarrow w\) is partially weakly covered in \(\D\) and removable in \(\G\).
  
  Now, instead assume that \(v \rightopen w\) is removable in \(\G\).
  We will show that \(\mt_{\G_{- v \rightopen w}}(v) \subseteq \mt_{\G_{- v \rightopen w}}(w)\).
  There are three cases:
  (i) \(v \rightopen w\) is non-essential (so it must be \(v \mathdash w\)) and both nodes have zero in-degree (i.e., no directed edges pointing to them),
  (ii) \(v \mathdash w\) but at least one of the nodes has nonzero in-degree, or
  (iii) the edge is essential, i.e., \(v \rightarrow w\).
  
  In case (i), because \(v \mathdash w\) is removable in \(\G\), there must exist some \(\D \in [\G]\) such that \(\D_{- v \openopen w}\) (we use \(\openopen\) to signify an edge in a DAG when the orientation is unknown) contains a trek between \(v\) and \(w\).
  Because both nodes have in-degree of 0 in \(\G\), this trek in \(\D\) must become an undirected path in \(\G_{- v \rightopen w}\), and thus \(\mt_{\G_{- v \rightopen w}}(v) = \mt_{\G_{- v \rightopen w}}(w) = \mt_\G(w)\).

  In case (ii), because \(\G\) is a CPDAG, it contains no induced subgraphs of the form \(p \rightarrow v \mathdash w\) (likewise with \(v\) and \(w\) swapped) \citep[Theorem 4.1]{AMP97}, and so \(\pa_\G(v) = \pa_\G(w) = \pa_\G(\cc_\G(\{v, w\}))\).
  Thus, \(\mt_{\G_{- v \rightopen w}}(v) = \mt_{\G_{- v \rightopen w}}(w) = \mt_\G(w)\).

  Finally, in case (iii), because \(v \rightopen w\) is removable in \(\G\), there must exist some \(\D \in [\G]\) such that \(\ma_{\D_{- v \rightarrow w}}(v) \subseteq \ma_{\D_{- v \rightarrow w}}(w)\) (either because \(v\rightarrow w\) is implied by transitivity or is partially weakly covered in \(\D\)).
  In other words, for every \(v' \in \pa_\D(v)\) either also \(v' \in \pa_{\D_{- v \rightarrow w}}(w)\) or there is a trek between \(v'\) and some \(w' \in \pa_{\D_{- v \rightarrow w}}(w)\), i.e., we have the trek \(v \leftarrow v' \leftarrow \cdots \leftarrow t \rightarrow \cdots \rightarrow w' \rightarrow w\) (with \(v', t,\) and \(w'\) not necessarily distinct).
  For each such induced trek in \(\D\), consider the path along those same nodes in \(\G\).
  Since \(\D \in [\G]\), edges along the path in \(\G\) either remain the same as in the trek in \(\D\) or become undirected, without the possibility for some \(\D' \in [\G]\) to have a collider along the induced path over the same nodes, i.e., each of these induced treks in \(\D\) becomes a partially directed trek in \(\G\) of the form \(v \cdots w' \rightarrow w\) or \(v \rightarrow w' \mathdash w\), and we have that \(\pa_\G(\cc(v)) = \pa_\G(v) \subseteq \pa_\D(v)\).
  Thus, \(\mt_{\G_{- v \rightopen w}}(v) \subseteq \mt_{\G_{- v \rightopen w}}(w') \subseteq \mt_{\G_{- v \rightopen w}}(w)\), completing the proof.
\end{proof}
The result \(\G_{- v \rightopen w}\) of removing an edge from a CPDAG \(\G\) is not necessarily itself a CPDAG but instead is a PDAG with possibly multiple \textit{completions} (orientations) into different CPDAGs.
The following Theorem~\ref{thrm:completions} provides a criterion for testing if the PDAG is complete and a method for constructing all \textit{width-1} (unconditionally equivalent) completions of the PDAG otherwise, i.e., all completions \(\G'\) such that there exists a \(\D \in [\G]\) for which \(\D_{- v \openopen w} \in [\G']\) and \(\U^\D = \U^{\D'}\).

\begin{defn}
  An edge \(v \rightarrow w\) in a chain graph \(\G\) is called a (strong) \emph{protector} if it renders another arrow strongly protected, i.e., if it occurs in one of the four following configurations as an induced subgraph:
  \vspace{-2em}
  \tikzstyle{VertexStyle} = []
  \tikzstyle{EdgeStyle} = [->,>=stealth']
  \SetGraphUnit{2.2}
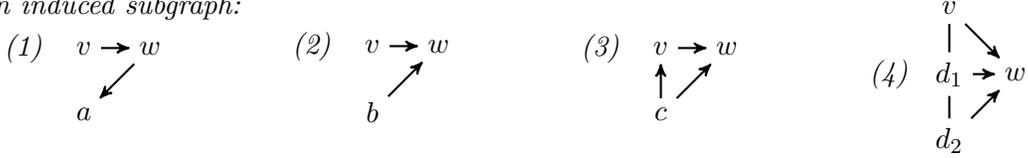
  \begin{figure}[h]
    \begin{subfigure}{.245\textwidth}
      \centering
      \begin{tikzpicture}[scale=.4]
        \node at (-2,0) {\textit{(1)}};
        \begin{scope}[execute at begin node=$, execute at end node=$]
          \Vertex{v} \EA(v){w} \SO(v){a}
          \Edges(v, w, a)
        \end{scope}
      \end{tikzpicture}  
    \end{subfigure}
    \begin{subfigure}{.245\textwidth}
      \centering
      \begin{tikzpicture}[scale=.4]
        \node at (-2,0) {\textit{(2)}};
        \begin{scope}[execute at begin node=$, execute at end node=$]
          \Vertex{v} \EA(v){w} \SO(v){b}
          \Edges(v, w) \Edges(b, w)
        \end{scope}
      \end{tikzpicture}  
    \end{subfigure}
    \begin{subfigure}{.245\textwidth}
      \centering
      \begin{tikzpicture}[scale=.4]
        \node at (-2,0) {\textit{(3)}};
        \begin{scope}[execute at begin node=$, execute at end node=$]
          \Vertex{v} \EA(v){w} \SO(v){c}
          \Edges(c, v, w) \Edges(c, w)
        \end{scope}
      \end{tikzpicture}  
    \end{subfigure}
    \begin{subfigure}{.245\textwidth}
      \centering
      \begin{tikzpicture}[scale=.4]
        \node at (-2,0) {\textit{(4)}};
        \begin{scope}[execute at begin node=$, execute at end node=$]
          \Vertex{d_1} \EA(d_1){w} \NO(d_1){v} \SO(d_1){d_2}
          \Edges(v, w) \Edges(d_1, w) \Edges(d_2, w)
          \tikzstyle{EdgeStyle} = []
          \Edges(v, d_1, d_2)
        \end{scope}
      \end{tikzpicture}  
    \end{subfigure}

    \label{fig:sole-protector}
  \end{figure}
  \newline
  Further, we say \(v \rightarrow w\) is a protector of a specific set of edges, e.g., it is a protector of \(\{w \rightarrow a\}\) in configuration (1) and a protector of \(\{d_1 \rightarrow w,\ d_2 \rightarrow w\}\) in configuration (4).
  Finally, \(v \rightarrow w\) is a \emph{sole protector} if at least one edge it protects has no other protector. 
\end{defn}

In the following, we let \(\ngh_\G(v)\) denote the open neighborhood of \(v\) in \(\G\) and  $\G_S$ denote the induced subgraph of $\G$ on vertices $S\subseteq V^\G$. 
A partially directed path $\langle u_0 = v,\ldots, u_m = w\rangle$ is called a \emph{leading trek} if $u_0\rightarrow u_1$ and all other edges on the path are undirected. 

\begin{thrm}\label{thrm:completions}
  Given an essential graph \(\G\) and removable edge \(v\rightopen w\), define \(T \coloneqq \ngh_\G (v) \cap \ngh_{\G_{\cc (w)}} (w)\).
  For the PDAG \(\G_{- v \rightopen w}\), we have the following:
  \begin{enumerate}
  \item if \(v \rightarrow w\) in \(\G\): the PDAG \(\G_{- v \rightarrow w}\) is a width-1 completion of \(\G\) if and only if \(|T| = 0\) and \(v \rightarrow w\) is not a sole protector;
  and its number of width-1 completions is \(C = \sum_{i=1}^m (2^{|K_i|}-1)\) if $\G_{- v\rightarrow w}$ contains only leading treks from $v$ to $w$, and \(C+1\) otherwise, where $K_1,\ldots, K_m$ are the maximal cliques in $\G_T$.
  \item if \(v \mathdash w\) in \(\G\): the PDAG \(\G_{- v \mathdash w}\) is a width-1 completion of \(\G\) if and only if \(|T| \leq 1\);
    and it has \(|T| + \mathbf{1}_{\mathrm{indegree}(v)}\) width-1 completions otherwise.
 \end{enumerate}
\end{thrm}

\begin{proof}
  For characterization 1, assume that \(\G_{- v \rightarrow w}\) is a width-1 completion of \(\G\).
  Notice that if there is a \(w'\) in \(\G\) such that \(w' \mathdash w\), then the induced subgraph \(v \rightarrow w' \mathdash w\) would be in \(\G_{- v \rightarrow w}\), contradicting that it is complete.
  Hence, there is no such \(w'\) and thus we have that \(\cc_{\G_{- v \rightarrow w}}(w)\) is empty and \(|T| = 0\).
Also, because \(\G_{- v \rightarrow w}\) is a CPDAG, all its directed edges are strongly protected, hence \(v \rightarrow w\) cannot be a sole protector in \(\G\).

  Conversely, assume that \(v \rightarrow w\) is a sole protector and \(|T| = 0\).
  Then, \(\G_{- v \rightarrow w}\) has at least one directed edge that is not strongly protected, so it is not complete.
  Furthermore, using \cite[Construction Algorithm]{AMP97} to iteratively undirect the resulting edges that become no longer solely protected produces exactly one width-1 completion.

  Alternatively, assume that $|T| > 0$. 
  Since $v\rightarrow w$ is removable in $\G$ there exists a DAG $\D\in[\G]$ such that $v\rightarrow w$ is removable in $\D$. 
  In fact, since any DAG in $[\G]$ in which $\D_{\cc(w)}$ is oriented with $w$ not a source node has $v\rightarrow w$ implied by transitivity in $\D$ then $v\rightarrow w$ is removable in all such DAGs. 
  Removing $v\rightarrow w$ from such a DAG $\D$ produces a DAG $\D_{-v\rightarrow w}$ that can only have v-structure $v\rightarrow t \leftarrow w$ not already in $\D$ where $t\in \ngh_\G(v)\cap \ngh_\G(w)$. 
  Since $\G$ is a chain graph it follows that either (i) $v\mathdash t \rightarrow w$ in $\G$, (ii) $v\rightarrow t\leftarrow w$ in $\G$ or (iii) $t\in T$. 
  Hence, the only v-structures in $\D_{-v\rightarrow w}$ not in $\D$ are given by case (ii) or case (iii).
  Those from the former case appear in all such DAGs, and hence only those from case (iii) distinguish different width-1 completions. 
  The number of ways to produce the v-structures in case (iii) correspond to the possible ways of orienting any non-empty subset of nodes in a maximal clique of a single connected component of $\G_{T}$ toward $w$ in $\G_{- v\rightarrow w}$ and all other edges adjacent to $w$ in $\G_{T\cup\{w\}}$ away from $w$.  
  By the MCS algorithm, any such configuration of v-structures is realized by some $\D_{-v\rightarrow w}$ for $\D\in[\G]$ with $w$ not a source node in $\D_{\cc(w)}$.

  Hence, if the only partially directed treks between $v$ and $w$ in $\G_{-v\rightarrow w}$ are leading treks from $v$ to $w$, then there are $C$ width-1 completions. 
  Otherwise, since \(v \rightarrow w\) is removable in \(\G\), we have that \(\mt_{\G_{- v \rightarrow w}}(v) \subseteq \mt_{\G_{- v \rightarrow w}}(w)\) and in \(\G_{- v \rightarrow w}\) either there are partially directed treks \(v \leftarrow \cdots a \cdots \rightarrow w\) for every \(a \in \mt(v)\) or there is at least one partially directed path \(v \mathdash \cdots \rightarrow w\).
  In either case, all DAGs in $[\G]$ have $v\rightarrow w$ removable.  
  Hence, the number of ways to produce v-structures in case (iii) correspond to all aforementioned placements of v-structures plus the configuration in which $v\rightarrow t \leftarrow w$ is a v-structure in $\G_{- v\rightarrow w}$ for all $t\in T$.
  Such a configuration is realized by initializing the MCS algorithm at $w$.
  Hence, we get $C+1$ width-1 completions, which concludes the proof of characterization 1.
 
  For characterization 2, assume that \(\G_{- v \mathdash w}\) is a width-1 completion of \(\G\).
  Then \(\cc_{\G_{- v \mathdash w}}(w)\) must be chordal, and so \(v \mathdash w\) must not be a chord in \(\G\).
  Hence, there must be less than two nodes in the intersection of \(v\)'s neighbors with \(w\)'s neighbors within the chain component, i.e., it must be that \(|T| \leq 1\) to avoid a chordless 4-cycle.
  
  Conversely, assume that \(|T| > 1\).
  For an arbitrary DAG \(\D \in [\G]\) with \(v \openopen w\) removable, consider \(\D_{- v \openopen w}\) which is Markov to \(\G_{- v \mathdash w}\).
  Note that \(\D\) has no v-structures in the induced subgraph over $\cc_\G(w)$ and that $\D_{- v \openopen w}$ contains a v-structure not in $\D$ if and only if it is of the form $v\rightarrow i \leftarrow w$ for some $i\in T$.
  Because $D_{- v \openopen w}$ is a DAG, every induced cycle in it must contain at least one v-structure. 
  Since we only removed the edge $v \openopen w$ from \(\D\), the only induced cycles of length at least $4$ in $D_{- v \openopen w}$ not in $D$ are $4$-cycles on vertex sets $\{v,w,p_1,p_2\}$ for $p_1,p_2\in T$.  
  If the number of $p\in T$ for which $v\rightarrow p\leftarrow w$ is a v-structure in $\D_{- v\openopen w}$ is less than $|T|-1$, then $D_{-v\openopen w}$ contains an induced $4$-cycle with no v-structure, which contradicts the fact that it is a DAG. 
  If the number of $v$-structures is $|T|$ and the indegree of \(v\) in \(\G\) is 0, then $v$ and $w$ are unconditionally independent in $\D_{-(v-w)}$, contradicting the assumption that $v \openopen w$ is removable in $\D$. 
  Hence, when \(\mathrm{indegree}_\G(v) = 0\), $\D_{- v \openopen w}$ must contain exactly $|T|-1$ v-structures of the form $v\rightarrow i \leftarrow w$ for $i\in T$, and these are all the v-structures in $\D_{-(v-w)}$ that are not in $\D$---and there are \(|T|\) ways of choosing these \(|T| -1\) v-structures, resulting in \(|T|\) possible width-1 completions of the PDAG when \(\mathrm{indegree}_\G(v) = 0\).
  Further, when \(\mathrm{indegree}_\G(v) > 0\), it is possible for \(\D_{- v \openopen w}\) to have \(|T|\) such v-structures, since \(v\) and \(w\) will remain \(d\)-connected by their common parent, in which case there will be \(|T| + 1\) possible width-1 completions of the PDAG, concluding the proof.
\end{proof}

\section{Discussion}
\label{sec:discussion}

Our study of unconditional equivalence has shown how comparatively simple undirected graphs can be informative of the underlying probabilistic or causal structure represented by DAGs.
Our transformational characterization of UECs (Theorem~\ref{thm: transformational}) generalizes that for MECs of DAGs \citep{chickering1995transformational} while also yielding a collection of moves analogous to those described in Meek's Conjecture, which were subsequently used to develop GES \citep{chickering2002optimal}.
Furthermore, the essential graph characterizations given in Theorems~\ref{thrm:essential} and~\ref{thrm:completions} allow for a more efficient traversal of the space of unconditionally equivalent MECs.

In an extended version of this paper, we will present a hybrid causal discovery algorithm that will first estimate a UEC and then search over it using the transformational moves from Theorem~\ref{thm: transformational}.  
Error propagation due to CI testing observed in classic hybrid methods is avoided with UEC estimation via independent pairwise independence tests. 
Other potential future work includes an extension of these results to characterizations of DAGs encoding the same CI relations with conditioning sets of size 0 or 1.  
Such results could be used to learn the 0-1 graphs studied by \citet{wille2006low}, who showed these models can be useful for estimating causal information in the small sample regime. 
Also of interest would be extensions of these results to ancestral graphs, as UECs are of particular use when modeling in the presence of latent confounders \citep{markham2022}.

\acks{Danai Deligeorgaki, Alex Markham, and Liam Solus were partially supported by the Wallenberg Autonomous Systems and Software Program (WASP) funded by the Knut and Alice Wallenberg Foundation.
  Pratik Misra was partially supported by the Brummer \& Partners MathDataLab.
  Liam Solus was partially supported the G\"oran Gustafsson Stiftelse and Starting Grant No.~2019-05195 from The Swedish Research Council.}

\bibliography{references}
\end{document}